\documentclass{article} 
\usepackage{iclr2026_conference,times}

\usepackage{amsmath,amsfonts,bm}

\def\eqref#1{equation~\ref{#1}}
\def\Eqref#1{Equation~\ref{#1}}

\def\1{\bm{1}}

\DeclareMathAlphabet{\mathsfit}{\encodingdefault}{\sfdefault}{m}{sl}
\SetMathAlphabet{\mathsfit}{bold}{\encodingdefault}{\sfdefault}{bx}{n}

\newcommand{\R}{\mathbb{R}}

\usepackage{hyperref}
\usepackage{url}

\usepackage[utf8]{inputenc} 
\usepackage[T1]{fontenc}    
\usepackage{hyperref}       
\usepackage{url}            
\usepackage{booktabs}       
\usepackage{amsfonts}       
\usepackage{nicefrac}       
\usepackage{microtype}      
\usepackage{wrapfig}
\usepackage{algorithm}
\usepackage{algpseudocode}
\usepackage{latexsym}
\usepackage{svg}
\usepackage{float}
\usepackage{graphicx}
\usepackage{mathrsfs}
\usepackage[scr=boondoxo]{mathalfa}
\usepackage{soul}
\usepackage{amsmath}
\usepackage{subcaption}
\usepackage{multirow}
\usepackage{amsthm}
\usepackage[normalem]{ulem}
\useunder{\uline}{\ul}{}

\newcommand{\Z}{\mathbb{Z}}
\newcommand{\W}{\textbf{W}}
\newcommand{\x}{\textbf{x}}

\newcommand{\gt}[1]{\mathscr{#1}}

\newcommand\abs[1]{\left|#1\right|}
\newcommand\dotprod[1]{\left\langle#1\right\rangle}
\newtheorem{theorem}{Theorem}

\newcommand{\method}{\text{AIRe}}

\iclrfinalcopy 

\title{Adaptive Training of INRs via \\ Pruning and Densification}

\author{%
    Diana Aldana$^{1}$\quad \textbf{João Paulo Lima}$^{1,2}$\quad Daniel Csillag$^3$ \quad Daniel Perazzo$^1$\\[0.1cm] \textbf{Haoan Feng}$^{4}$\quad \textbf{Luiz Velho}$^1$\quad \textbf{Tiago Novello}$^1$ \\[0.2cm]
    $^{1}${IMPA} \quad $^{2}${Universidade Federal Rural de Pernambuco}  \quad $^{3}${FGV EMAp} \quad $^{4}${University of Maryland}  %
}

\begin{document}

\maketitle

\begin{abstract}
Encoding input coordinates with sinusoidal functions into multilayer perceptrons (MLPs) has proven effective for implicit neural representations (INRs) of low-dimensional signals, enabling the modeling of high-frequency details.
However, selecting appropriate input frequencies and architectures while managing parameter redundancy remains an open challenge, often addressed through heuristics and heavy hyperparameter optimization schemes. 
In this paper, we introduce \method{} (\textbf{A}daptive \textbf{I}mplicit neural \textbf{Re}presentation), an adaptive training scheme that refines the INR architecture over the course of optimization. Our method uses a neuron pruning mechanism to avoid redundancy and input frequency densification to improve representation capacity, leading to an improved trade-off between network size and reconstruction quality.
For pruning, we first identify less-contributory neurons and apply a targeted weight decay to transfer their information to the remaining neurons, followed by structured pruning.
Next, the densification stage adds input frequencies to spectrum regions where the signal underfits, expanding the representational basis.
Through experiments on images and SDFs, we show that \method{} reduces model size while preserving, or even improving, reconstruction quality.
Code and pretrained models will be released for public use.
\end{abstract}
\vspace{-0.2cm}

\section{Introduction}
\label{sec:intro}
\vspace{-0.1cm}


Implicit neural representations (INRs) have emerged as a powerful framework for modeling low-dimensional signals -- such as images and signed distance functions (SDFs) -- by encoding them directly in the parameters of neural networks~\citep{sitzmann2020implicit, tancik2020fourier, saragadam2023wire, dam2025high}. Instead of storing signals discretely, INRs represent them as continuous functions, mapping input coordinates $\x$ to a network predicting the corresponding signal value. To capture high-frequency content, these networks typically employ two key components: (1) projecting $\x$ into a list of sinusoidals $\sin(\omega\x+\varphi)$, where $\omega$ and $\varphi$ denote the input frequencies and phase shifts, and (2) using periodic activation functions throughout the network layers. This combination enables INRs to represent fine details that standard ReLU-based MLPs struggle to learn due to their spectral bias~\citep{tancik2020fourier,sitzmann2020implicit}.

\begin{figure}[t]
    \centering
    \includegraphics[width=\linewidth]{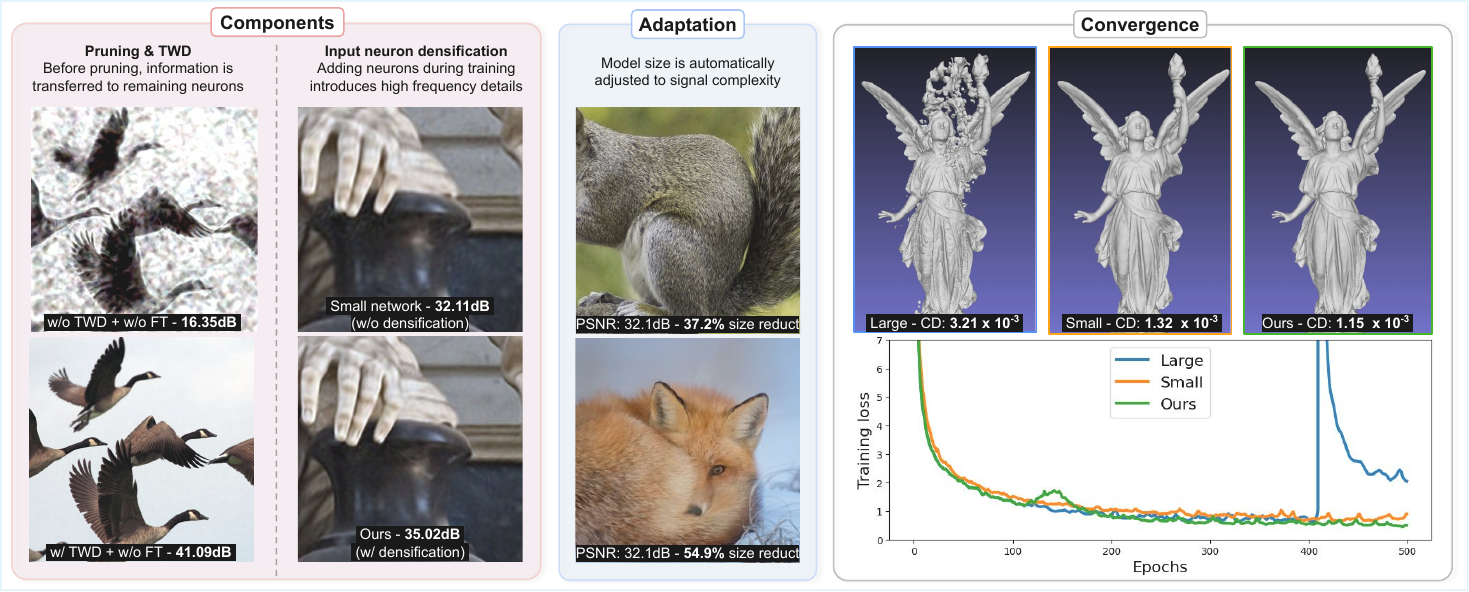}
    \caption{
    We present \method{}, a robust training method that adaptively fits the INR architecture to the target signal
    through two complementary mechanisms: (i) pruning with targeted weight decay (TWD)
    which mitigates parameter redundancy and fine tuning (FT) dependence by transferring information prior to structured neuron removal (see birds), and (ii) input frequency densification, which augments the representation basis, enhancing convergence and details fidelity (see hand). 
    We compare three strategies: (i) an overparameterized SIREN model with standard training (\textit{large network}), (ii) a model adapted with \method{} (\textbf{Ours}), and (iii) a \textit{small network} fitted with standard training. \method{} improves reconstruction accuracy while producing more compact networks (blue box), and enhances training convergence in settings where overparameterization leads to divergence (see statue box). 
    }
    \label{fig:teaser}
    \vspace{-0.2cm}
\end{figure}

Choosing an appropriate network architecture and input frequencies $\omega$ to accurately and compactly fit a target signal is a challenging task. 
Most prior work has addressed this by enhancing the expressiveness of INRs via tailored initialization schemes and specialized activation functions. For example, \citet{zell2022seeing} leveraged an initialization based on Fourier series to control the network's spectrum, enhancing its ability to represent fine-grained details. TUNER~\citep{novello2024taming} provided a theoretical justification for this approach and introduced a training procedure to bandlimit the spectrum dynamically. FINER~\citep{liu2024finer}, on the other hand, employed a modified sine activation combined with bias initialization, allowing the modeling of high-frequency components.
Despite these advances, selecting a compact yet expressive architecture a priori remains difficult: undersized networks tend to underfit the data, while oversized ones often lead to training instabilities and increased susceptibility to overfitting.

To address this challenge, we introduce \textbf{\method} (\textbf{A}daptive \textbf{I}mplicit neural \textbf{Re}presentation), a training framework that progressively adapts a potentially overparametrized INR to the target data through two complementary operations: \emph{pruning} and \emph{densification} of neurons. 
For pruning, we evaluate the contribution of each neuron using a customizable criterion (e.g. weight norms) to identify the most redundant ones. 
To transfer information from these low-contributing neurons to more relevant ones, we propose a novel \textit{targeted weight decay} (TWD) mechanism, which penalizes their weights prior to structured removal.
Once this transfer is induced, the targeted neurons are pruned.
For densification,
we introduce new input
frequencies in underfit regions of the spectrum, expanding the network’s representational capacity when necessary.
By dynamically aligning model complexity with the input data, \method~finds compact INRs that accurately reproduce the target signal.
We showcase some of \method{}'s results in Figure~\ref{fig:teaser},
illustrating strong performance in reconstruction quality, model compactness, and training stability.
\textbf{Our main contributions~are:}
\begin{itemize}
    \item A general framework for the adaptive training of INRs, driven by pruning and densification. The pruning component brings principles from neural network pruning to the INR setting, while also introducing a novel targeted weight decay (TWD) strategy to preserve quality during neuron removal (see Figure~\ref{fig:prun_evol_cols}).
    For densification, we add new input frequencies in underfit spectral regions, enhancing representational capacity (Table~\ref{tab:prune-densify-ablation}).
    Combined, these components enable accurate signal fitting with compact, data-adaptive architectures (Table~\ref{tab:image_sdf_nerf}).

    \item  A theoretical analysis of both pruning and densification mechanisms for INRs.
    In particular, we leverage a harmonic expansion of sinusoidal neural networks (Theorem~\ref{t-expansion}) to derive principled densification schemes, and prove stability of our neural networks under magnitude-based pruning (Theorem~\ref{p-criterion}).
    Together, these promote densification and pruning mechanisms that mitigate divergence during training (cf. Figure~\ref{fig: exp_surfaces_divergence}).

    \item  An empirical evaluation of \method{} across a range of image fitting and 3D shape reconstruction benchmarks. We show that \method{} consistently outperforms both the standard neural network training pipeline (see Table~\ref{tab:image_sdf_nerf}) as well as recent adaptive training methods (Table~\ref{tab:sota}) in terms of the accuracy-efficiency trade-off.
\end{itemize}

\section{Related work}
\label{sec:related_works}

\textbf{INRs} emerged as a modern paradigm for learning low-dimensional signals such as images \citep{chen2021learning, shi2024improved}, image morphing~\citep{schardong2023neural, bizzi2025flowing}, SDFs~\citep{yang2021geometry, novello2022exploring, schirmer2024geometric}, displacement fields~\citep{yifan2021geometry}, surface animation~\citep{mehta2022level, novello2023neural}, and multiresolution signals  \citep{paz2023mr, saragadam2022miner, lindell2022bacon, wu2023neural}.
On the methodological side, several works have explored the representation capacity of INRs \citep{mehta2021modulated, yuce2022structured, saratchandran2024sampling}, as well as the critical role of initialization strategies \citep{novello2022understanding, paz2024implicit, saratchandran2024activation, finn2017model,yeom2024fast}.

\textbf{Neural network pruning} has long been of interest to the machine learning community \citep{lecun1989optimal, hassibi1993optimal, thimm1995evaluating, frankle2018the, hoefler2021sparsity, blalock2020state, menghani2023efficient}. 
Classic approaches have relied on metrics such as weight magnitude, salience, or second-order derivatives, and are often followed by fine-tuning or regularization (e.g., weight decay) to preserve performance~\citep{han2015learning, tessier2022rethinking}.
However, it is known that methods often fail to generalize beyond their original settings \citep{blalock2020state}.
To the best of our knowledge, \citet{zell2022seeing} is the only prior work exploring the pruning (or adaptation) of INRs.
Their method removes input neurons to select an appropriate representational basis, but they did not explore hidden layer pruning.
In contrast, our method adapts the model size to target redundancy in the signal detail content while choosing a fitting input frequency encoding.

Recent work has investigated ways to adapt network architectures during training.
The lottery ticket hypothesis~\citep{frankle2018the} suggests that sparse subnetworks within overparameterized models can perform just as well when trained independently.
Building on this idea, RigL~\citep{evci2020rigging} dynamically adjusts connectivity by pruning and growing connections during training.
While promising, such strategies have not been studied in the context of INRs, where the objectives, data modalities, and inductive biases differ significantly from those in standard classification tasks.
In Table~\ref{tab:sota}, we adapt these methods to the INR setting and compare them with \method{}, showing that our approach achieves superior results.

\section{Adaptive training of INRs}
\label{sec:method}
\vspace{-0.1cm}
\label{s-adaptive-training}


\begin{figure*}[h!]
    \centering
    \includegraphics[width=\textwidth]{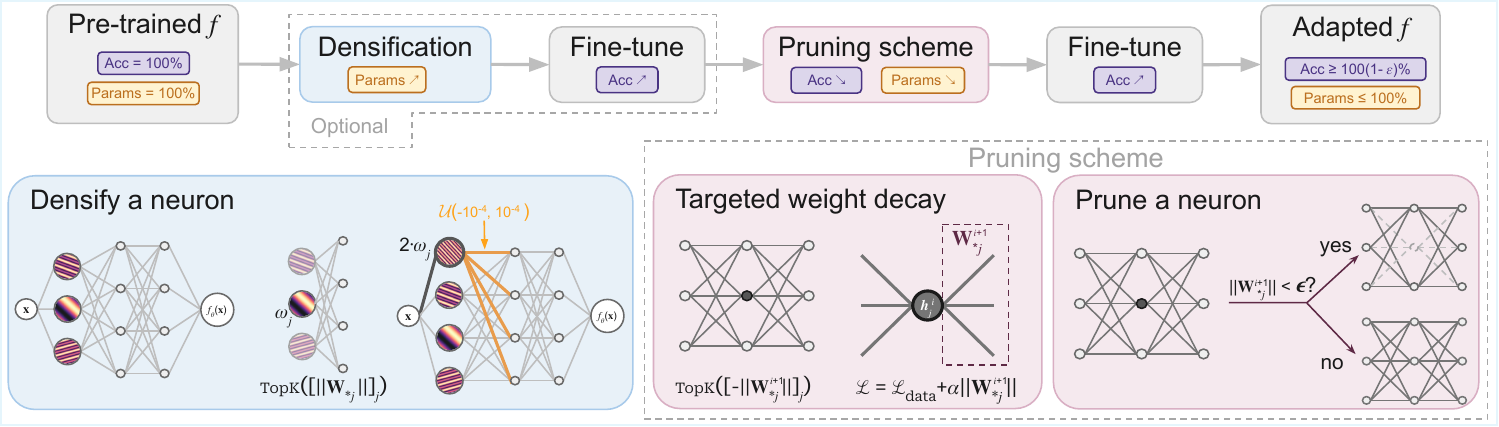}
    \caption{
    We present \method, a training framework that adapts network architecture through two theoretically grounded strategies: \textit{densification} and \textit{pruning}. For signals with rich frequency content, densification selects the most relevant input frequencies $\omega_j$ and expands the spectrum by augmenting $\omega$ with $2 \cdot \omega_j$.
    To reduce network size, pruning identifies candidate neurons via magnitude criterion, transfers information during training with a novel targeted weight decay (TWD) regularization, and removes neurons whose norm falls below a threshold $\epsilon$. The function $\text{TopK}(v)$ selects the $K$ largest entries of $v$.
    }
    \label{fig: overview}
\end{figure*}


Our goal is to develop a training framework that dynamically adapts a sinusoidal INR architecture to the given data samples $\{\x_j, \gt{f}_j\}$ from a low-dimensional signal $\gt{f}$.
Specifically, we want to adjust the size of a
sinusoidal MLP
of depth $d \in \mathbb{N}$ defined as $f(\mathbf{x}) = \mathbf{L} \circ \mathbf{S}^d \circ \cdots \circ \mathbf{S}^0(\mathbf{x})$, a composition of $d$ sinusoidal layers $\mathbf{S}^i(\mathbf{x}) = \sin(\mathbf{W}^i \mathbf{x} + \mathbf{b}^i)$ parameterized by a weight matrix $\mathbf{W}^i \in \mathbb{R}^{n_{i+1} \times n_i}$ and a bias vector $\mathbf{b}^i \in \mathbb{R}^{n_{i+1}}$, followed by an affine layer $\mathbf{L}$.
Observe that the first layer $\mathbf{S}^0$ maps the input coordinates $\mathbf{x}$ into a harmonic embedding of the form $\sin(\omega \mathbf{x} + \varphi)$, where we denote $\omega := \mathbf{W}^0$ as the matrix of \emph{input frequencies} and $\varphi := \mathbf{b}^0$ as the vector of \emph{phase shifts}. 

Although the choice of $\{n_i\}_i$ is critical for determining network capacity, it is typically based on empirical heuristics.
Moreover, a model with poorly initialized input frequencies $\omega$ may fail to capture the full spectrum of the signal, leading to unsatisfactory reconstruction.
To address these problems, we adapt a model architecture by adding and removing neurons. More precisely, we define the $ij$-\textit{neuron} $h_j^i(\textbf{x})$ of $f$ as the $j$-th coordinate of the output of the $i$-th layer, that is, 
\begin{align}\label{eq: neuron}
h_j^{i+1}(\textbf{x}) = \sin\bigl(\W_{j*}^{i+1} \sin(\textbf{y}^{i}) + b_j^{i+1}\bigr),    
\end{align}
where $\textbf{y}^{i}$ denotes the linear transformation of the $i$th layer prior to activation. 
Then, we densify the input layer by appending new neurons to $\textbf{h}^0(\textbf{x})$, introducing novel frequencies to expand the spectral coverage.
Finally, we employ a magnitude-based neuron pruning scheme to account for potential redundancy in parameters.
Figure~\ref{fig: overview} provides an overview of \method.


\subsection{Densification}

Sinusoidal INRs employ an encoding layer to mitigate spectral bias and enhance the representation of high-frequency signals.
However, they are heavily dependent on their initialization, which may lead to noisy reconstructions or slower training.
Here, we propose a principled input neuron densification that aims to improve reconstruction quality of highly detailed signals.

To do so, we must analyze the spectrum of an INR. This can be done by a theorem of \cite{novello2024taming}, which provides a trigonometric expansion that facilitates this analysis.


\begin{theorem}\label{t-expansion}
    The neuron $h_j^{i+1}$ admits the following amplitude-phase expansion:
    \emph{\begin{align}\label{e-neuron-expansion}
        h_j^{i+1}\!({\x})\! =\!\! \sum_{\textbf{k} \in \mathbb{Z}^{n_i}}\!\! \alpha_{\textbf{k}}(\W^i_{j*}) \, \sin\!\left( \langle \textbf{k}, \textbf{y}^i \rangle + b_j^{i+1} \right),\quad \text{where}\quad 
        \abs{\alpha_{\textbf{k}}(\W^i_{j*})}\!\leq \prod_{l}\frac{1}{|k_l|!}\left(\! \frac{|W^{i+1}_{jl}|}{2} \!\right)^{|k_l|}
    \end{align}}
Here, \emph{$\alpha_\textbf{k}(\W^i_{j*})=\prod_{l}J_{k_l}(W_{jl}^{i+1})$} is the product of Bessel functions.
\end{theorem}

This result shows that the composition of sinusoidal layers generates new frequencies of the form $\langle \textbf{k}, \omega \rangle$, depending solely on the input frequencies $\omega$, with phase shifts determined by the biases $\{\varphi, \textbf{b}^i\}$.
Additionally, the amplitudes $\alpha_{\textbf{k}}$ depend exclusively on the hidden weight matrices $\textbf{W}^i$.
Thus, the generated frequencies are governed by the input embedding, while the hidden parameters control the amplitudes and phase shifts.
Moreover, from \Eqref{e-neuron-expansion} we observe that $$\textbf{h}^0(\textbf{x})=\left[\sum_{\textbf{k}\in\Z^{n_0}}\alpha_\textbf{k}\sin\left(\langle\textbf{k}, \omega\rangle\textbf{x} + b_j \right)\right]_j \text{ with bias } b_j =\langle\textbf{k}, \varphi\rangle + b^1_j.$$ Thus, adding an input neuron with frequency $\omega'$ expands the layer spectrum from $\{\langle\textbf{k}, \omega\rangle\}_\textbf{k}$ to $\{\langle\textbf{k}, \omega\rangle+l\cdot\omega'\}_{\textbf{k}, l}$.
Since the frequencies in the input layer determine those appearing in the network, the densification of the input layer greatly increases the expressiveness of the overall network.

However, identifying new frequencies to be added is fairly nontrivial.
Fortunately, Theorem~\ref{t-expansion} also sheds light on this:
the $j$-th column of $\W^1$ influences the value of any amplitude $\alpha_\textbf{k}$ related to the generated frequency $\textbf{k}\cdot\omega$, with $k_j\neq0$. In particular, let us consider the case of $\textbf{k}=\textbf{e}_j$, where $\mathbf{e}_j$ denotes the $j$-th canonical basis vector. If $\|\W^1_{*j}\|$ is small, then by standard properties of Bessel functions $\alpha_{\textbf{e}_j}=J_1(W^1_{ij})$ must also be small and $\alpha_{2\textbf{e}_j}=J_2(W^1_{ij})$ is negligible. Conversely, when $\|\W^1_{*j}\|$ is large, $\alpha_{2\textbf{e}_j}$ carries non-negligible energy and the generated frequency $\dotprod{2\textbf{e}_j, \omega} = 2\omega_j$ may contribute to the reconstruction of the target signal.
However, for it to indeed strongly influence reconstruction, the values of $\W^i_{*j}$ must increase, which happens slowly. So, to accelerate the training of such frequencies, we first identify highly contributing neurons by assessing the magnitudes of their weights and initialize novel input frequencies accordingly; to be precise, for every highly contributing $\omega_j$ we introduce a new frequency $2\omega_j$, enabling it to influence the network spectrum more easily.

The corresponding new column in the hidden matrix $\W^1$ is initialized with random values drawn from a uniform distribution in the range $[-10^{-4},10^{-4}]$, ensuring a stable start for training.
Finally, the network is retrained to fine-tune all parameters, allowing it to adapt to the extended frequency spectrum and fully leverage the increased representational capacity.

\subsection{Pruning}

Determining an appropriately sized model capable of representing the target signal with quality is a key challenge when training sinusoidal INRs.
Typically, large architectures are employed to ensure reconstruction accuracy, sacrificing model compactness.
To address this, we design a pruning procedure that detects and removes redundant neurons during training.

First, we employ a magnitude-based criterion to identify uninformative neurons, a common strategy in classical network pruning. We now provide a formal justification of its validity for INRs: in sinusoidal MLPs, pruning neurons induces only a bounded perturbation to the overall function. This perturbation depends on the $\infty$-operator norms of the parameter changes and the norms of the subsequent layers.

\begin{theorem}
\label{p-criterion}
Let $f$ be a sinusoidal INR of depth $d$, and let $\widetilde{f}$ be the network obtained by perturbing the $k$-th hidden layer weights and biases to $\widetilde{\mathbf{W}}^k$ and $\widetilde{\mathbf{b}}^k$. Then,
\begin{align*}
    \sup_{x} \left\| f(x) - \widetilde{f}(x) \right\|_\infty
    \leq \left( \left\|\mathbf{W}^k - \widetilde{\mathbf{W}}^k\right\|_\infty + \left\|\mathbf{b}^k - \widetilde{\mathbf{b}}^k\right\|_\infty \right) \|\mathbf{L}\|_\infty \prod_{i=k+1}^d \|\mathbf{W}^i\|_\infty.
\end{align*}
\end{theorem}

Theorem~\ref{p-criterion} formally guarantees that small modifications to a layer’s parameters---such as pruning neurons with small outgoing weights---induce only proportionally small changes to the network’s output.
This justifies magnitude-based pruning both intuitively and theoretically. 
However, training directly with the reconstruction loss $\mathcal{L}_\text{data}$ (e.g. MSE) often leads to relatively few truly redundant neurons, even in overparametrized architectures. For better pruning, we employ a targeted weight decay (TWD) strategy that reduces the contribution from near-redundant neurons, turning them truly redundant.
It consists of training the network $f$ with the loss function, 
\begin{align}
    \mathcal{L}_{\alpha, \mathcal{I}} = \mathcal{L}_{\text{data}} + \alpha\sum_{j\in\mathcal{I}}\|\W^{i+1}_{*j}\|_1, \quad \text{with} \quad \alpha \in [0,1),
\end{align}
where $\mathcal{I}=\text{TopK}\left(\left[-\|\W^{i+1}_{*j}\|_1\right]_j\right)$ are the $K$ indices of the neurons with the smallest column norm.

Our procedure uses TWD to isolate low-impact neurons, ensuring that pruning remains consistent with the theoretical stability.
Then, we select the neurons to remove by thresholding small $\ell_1$ norms, e.g., pruning $h^i_j$ if $\|\textbf{W}^{i+1}_{*j}\|_1 \!=\! \|\textbf{W}^{i+1} - \widetilde{\textbf{W}}^{i+1}\|_{\infty} \!\leq\! \epsilon$ (where $\widetilde{\textbf{W}}^{i+1}$ denotes the altered weight~matrix), and fine-tune the network to recover performance.
As illustrated in Figure~\ref{fig: overview}, pruning a neuron $h^i_j(\textbf{x})$ involves removing its outgoing connections.
In practice, we mask only the entries of the $j$-th column $\W^{i+1}_{*j}$, which implicitly leaves unused the row $\W^i_{j*}$ and bias $b^i_j$. 
Note that pruning the input layer may have greater impact on the reconstruction since we are deleting an input frequency; that is, we are eliminating many generated frequencies from the network spectrum.

\section{Experiments}
\label{sec:experiments}

We evaluate \method{} on adaptive training across three tasks: image fitting, surface reconstruction (SDFs), and novel view synthesis with NeRFs. Experiments are conducted on the DIV2K~\citep{agustsson2017div2k}, Stanford Repository~\citep{curless1996volumetric}, and NeRF Synthetic~\citep{mildenhall2021nerf} datasets. 
We also study \method{} in a setup where the final architecture is fixed, demonstrating that our training procedure can improve reconstruction quality even when the reduced small architecture is known in advance. 
Finally, we perform ablation studies to validate the design choices underlying our method.

All models are implemented in PyTorch~\citep{paszke2019pytorch} and optimized with Adam~\citep{kingma2015adam}.
For simplicity, we denote a sinusoidal MLP architecture by $[n_1, ..., n_{d+1}]$, where $d$ is the number of hidden layers and $n_i$ is the number of neurons in the $i$-th layer.


\paragraph{Comparison with standard training.} We compare \method{} against a baseline defined by the original, large initial architecture (overparametrized) trained with the standard neural network training pipeline, showing that \method{} can reduce model size while maintaining reconstruction quality by finding more appropriate input frequencies.
We evaluate this on images, SDFs, and NeRFs, adopting commonly used architectures for each task (SIREN and FINER).
Table~\ref{tab:image_sdf_nerf} shows that \method{} achieves substantial reductions in model size while maintaining reconstruction quality, and in several cases even improving~it.
\renewcommand{\arraystretch}{1.1}
\begin{table}[h!]
\centering
\setlength{\tabcolsep}{3pt}
\footnotesize
\caption{\textbf{Our method fits a compact INR to the target signal while preserving accuracy.}
We evaluate \method{} (`Ours') against an overparametrized INR trained with the standard training pipeline (`Large') on images (with model size [512, 256, 256]), SDFs (with architecture $[256, 256, 256]$), and NeRF  tasks (with size $[256, 128, 128]$), reporting PSNR and Chamfer Distance ($\times 10^2$).
\method{} enables a strong reduction in model size, while preserving or even improving quality.
}
\label{tab:image_sdf_nerf}
\begin{tabular}{llll|llll|llll}
\hline
\textbf{\begin{tabular}[c]{@{}l@{}}Imgs\\ \textsuperscript{(Div2K)}\end{tabular}} &
  \multicolumn{1}{l}{\textbf{\begin{tabular}[c]{@{}l@{}}Large \\ PSNR\end{tabular}}} &
  \multicolumn{1}{l}{\textbf{\begin{tabular}[c]{@{}l@{}}Ours \\ PSNR\end{tabular}}} &
  \multicolumn{1}{l|}{\textbf{\begin{tabular}[c]{@{}l@{}}Size\\  reduct.\end{tabular}}} &
  \textbf{\begin{tabular}[c]{@{}l@{}}SDFs\\ \textsuperscript{(Stanford)}\end{tabular}} &
  \multicolumn{1}{l}{\textbf{\begin{tabular}[c]{@{}l@{}}Large\\ CD\end{tabular}}} &
  \multicolumn{1}{l}{\textbf{\begin{tabular}[c]{@{}l@{}}Ours \\ CD\end{tabular}}} &
  \multicolumn{1}{l|}{\textbf{\begin{tabular}[c]{@{}l@{}}Size\\ reduct.\end{tabular}}} &
  \textbf{\begin{tabular}[c]{@{}l@{}}NeRF\\ \textsuperscript{(Synthetic)}\end{tabular}} &
  \multicolumn{1}{l}{\textbf{\begin{tabular}[c]{@{}l@{}}Large\\ PSNR\end{tabular}}} &
  \multicolumn{1}{l}{\textbf{\begin{tabular}[c]{@{}l@{}}Ours\\ PSNR\end{tabular}}} &
  \multicolumn{1}{l}{\textbf{\begin{tabular}[c]{@{}l@{}} Size\\ reduct.\end{tabular}}} \\ \hline
\#00 & 31.96 & 31.56 & 35.89\% & Armadillo   & 0.62 & 0.63 & 73.30\% & Lego      & 25.72 & 25.30  & 35.00\% \\
\#01 & 37.93 & 35.63 & 65.28\% & Bunny       & 0.76 & 0.71 & 72.33\% & Materials & 23.71 & 23.62 & 30.86\% \\
\#02 & 30.76 & 29.17 & 52.39\% & Dragon      & 0.73 & 0.61 & 70.33\% & Ficus     & 24.23 & 24.82 & 26.83\% \\
\#03 & 37.40  & 35.04 & 56.80\% & Buddha & 0.59 & 0.56 & 41.27\% & Hotdog    & 29.69 & 28.68 & 33.14\% \\
\#04 & 33.88 & 31.09 & 60.08\% & Lucy        & 0.92 & 0.58 & 52.50\% & Drums     & 22.17 & 22.02 & 39.09\% \\ \hline
\end{tabular}
\end{table}

\pagebreak

\begin{wraptable}[9]{r}{6.7cm}
\centering
\small
\caption{\textbf{Comparison of pruning criteria.} Results are on the image representation task.}
\label{tab:sota}
\begin{tabular}{l|ll}
\hline
\textbf{Method} & \textbf{PSNR↑} & \textbf{SSIM↑} \\ \hline
Baseline & 34.60 ± 3.82          & 0.92 ± 0.03 \\
DepGraph & 27.56 ± 2.12          & 0.82 ± 0.04 \\
RigL     & 34.29 ± 3.37          & \textbf{0.95 ± 0.01} \\
AIRe (ours)     & \textbf{37.07 ± 3.74} & \textbf{0.95 ± 0.01} \\ \hline
\end{tabular}
\end{wraptable}
\paragraph{Comparison against existing pruning baselines} are provided in Table~\ref{tab:sota}, for the task of image representation using the same configuration as in Table~\ref{tab:image_sdf_nerf}, with a SIREN architecture. 
For this comparison, we consider two model-agnostic pruning methods with publicly available implementations, DepGraph~\citep{fang2023depgraph} and RigL~\citep{evci2020rigging}, as well as a baseline given by training a reduced architecture from scratch with standard training. The pruning rate of each method is set to approximately 25\% of the original parameters, and we follow the hyperparameter choices reported in the respective papers.
%
AIRe consistently outperforms these pruning methods, demonstrating its effectiveness for INR architecture adaptation over training.

\subsection{\method{} vs. small networks}

\method{} starts with a \textbf{large} architecture and progressively reduces its size during training, resulting in a \textbf{small} network. To evaluate how effectively \method{} leverages its architectures, we compare it against standard training applied directly to both the initial (large) architecture and the final (small) one. We conduct this evaluation for image fitting (DIV2K) and SDF reconstruction (Stanford Repository).

For the SDF reconstruction task, we follow the implementation in \citep{novello2022exploring}, training each network for $10^3$ epochs, sampling $10^4$ on-surface points and $10^4$ off-surface points uniformly.
Meshes are extracted from the trained SDFs via marching cubes with a resolution of $512^3$, and all surfaces were normalized to $[-1,1]^3$. For evaluation, we report the number of network parameters (Params) and the Chamfer Distance (CD) between reconstructed and ground-truth surfaces. We also evaluate \method{} without densification, as SDFs typically contain less details than other applications.
\renewcommand{\arraystretch}{1.1}
\begin{table}[h!]
\centering
\setlength{\tabcolsep}{3pt}
\footnotesize
\caption{\textbf{\method{} vs. directly trained large and small networks.} 
We compare \method{} with standard training applied to large and small architectures on both SDF reconstruction (Stanford) and image fitting (DIV2K). 
Metrics are CD ($\times 10^2$) for SDFs and PSNR for images, along with parameter reduction relative to the large model. 
\method{} achieves accuracy comparable to or better than the large network while using the same reduced parameter budget as the small one.}
\label{tab: exp_surfaces}
\begin{tabular}{lllc|llcc}
\hline
\textbf{\begin{tabular}[c]{@{}l@{}}Model\\ \textsuperscript{(SDFs)} \end{tabular}} &
  \textbf{Variant} &
  \textbf{CD (×10²) ↓} &
  \textbf{\begin{tabular}[c]{@{}l@{}}Size\\ reduct. ↓\end{tabular}} &
  \textbf{\begin{tabular}[c]{@{}l@{}}Model\\ \textsuperscript{(Images)} \end{tabular}} &
  \textbf{Variant} &
  \textbf{PSNR $\uparrow$} &
  \textbf{\begin{tabular}[c]{@{}l@{}}Size\\ reduct. ↓\end{tabular}} \\ \hline
\multirow{3}{*}{\begin{tabular}[c]{@{}l@{}}SIREN\\ \end{tabular}}  & Large & 0.65 ± 0.11 & -       & \multirow{3}{*}{\begin{tabular}[c]{@{}l@{}}SIREN\\ \end{tabular}}  & Large & \textbf{39.59 ± 3.30} & -        \\
                       & Small & 0.89 ± 0.09 & 83.96\% &                      & Small & 34.60 ± 3.82  & 24.95\%  \\
                       & Ours  & \textbf{0.64 ± 0.03} & 83.96\% &             & Ours  & 37.07 ± 3.74 & 24.95\%  \\[.15cm]
\multirow{3}{*}{\begin{tabular}[c]{@{}l@{}}FINER\\ \end{tabular}} & Large & 2.14 ± 0.41 & -        & \multirow{3}{*}{\begin{tabular}[c]{@{}l@{}}FINER\\ \end{tabular}} & Large & 38.77 ± 2.98 & - \\
                       & Small & 5.08 ± 3.51 & 83.96\% &                      & Small & 38.87 ± 3.44 & 24.95\%  \\
                       & Ours  & \textbf{0.88 ± 0.15} & 83.96\% &             & Ours  & \textbf{39.91± 3.89}  & 24.95\%        \\
\hline
\end{tabular}
\end{table}

For training, we start with a large network $[256, 256, 256]$, trained from scratch for 200 epochs. We then select $192$ neurons from both hidden layers and continue training with targeted weight decay (TWD) for 500 epochs. Finally, the selected neurons are pruned, and the resulting smaller network $[64, 64, 256]$ is retrained for 300 epochs.
Table~\ref{tab: exp_surfaces}(left) shows that \method{} provides a better SDF reconstruction than the large network in all cases.
We also highlight that our approach obtains similar or better accuracy compared to the initial, large network trained from scratch while using roughly between a third and a sixth of the network parameters for surface representation.
Aditionally, we found that AIRe has a comparable time overhead ($76.8s$) compared to the large SIREN model ($76.0s$). Moreover, even when training a small architecture during $83.2s$, it performs worse than \method{} with $0.86\times10^2$ for CD metric.  
Figure~\ref{fig: exp_surfaces} shows some qualitative comparisons of \method{} and the small architecture with standard training on the Armadillo, Buddha, and Lucy models, showing that \method{} provides, in general, a lower (bluer) distance from the ground-truth surface.

\begin{figure}[h!]
    \centering
    \includegraphics[width=\linewidth]{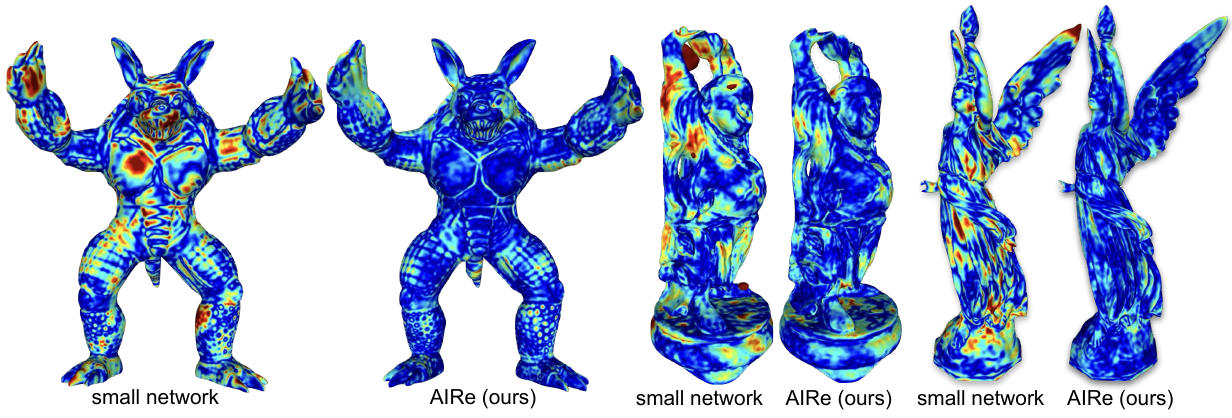}     
    \caption{Qualitative comparison of SDF reconstructions on the Armadillo, Buddha, and Lucy models using a SIREN with $\omega_0=60$ and small network size $[64, 64, 256]$. Left: results of training the final small network directly. Right: results of \method{}. Colors indicate the distance from the ground-truth surface, from dark blue (0) to dark red ($\geq 0.01$). \method{} produces reconstructions that are consistently closer to the ground truth than those obtained by training the small network from scratch.}
    \label{fig: exp_surfaces}
\end{figure}

Additionally, AIRe mitigates divergence during the training of SDF models, as illustrated in Figure~\ref{fig: exp_surfaces_divergence}. We illustrate this by initializing a large network of size $[256, 256, 256, 256]$ and training it on the Armadillo for half the epochs with $\omega_0 = 60$ and small network architecture $[64, 64, 256]$. Under standard training, the large network diverges, producing reconstructions with severe noise and artifacts. In contrast, \method{} yields a more accurate reconstruction despite using less than half the parameters of the initial model.
\begin{figure}[h!]
    \centering
    \includegraphics[width=\linewidth]{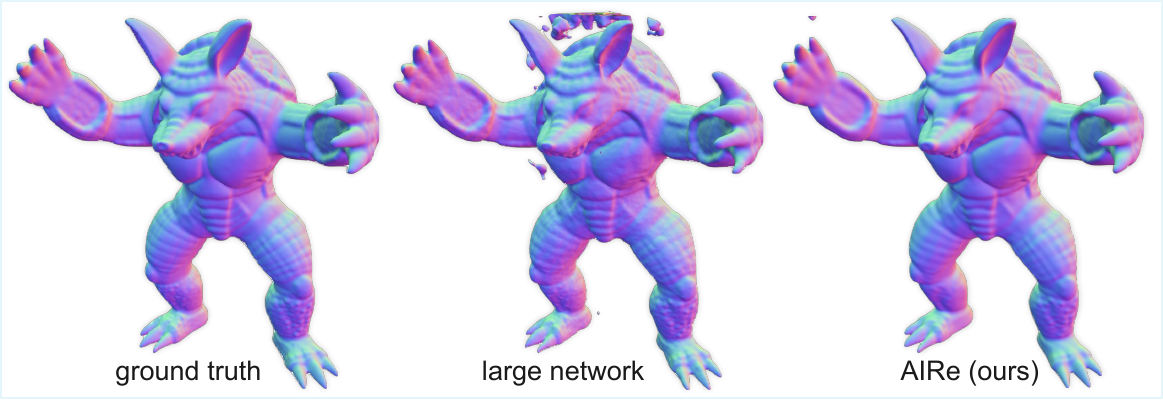}     
    \caption{Qualitative comparison on the Armadillo. Left: ground truth. Middle: standard training of the large network, which diverges and produces noisy artifacts. Right: \method{}, which avoids divergence and yields a cleaner reconstruction with fewer parameters.}
    \label{fig: exp_surfaces_divergence}
\end{figure}

\pagebreak

For the evaluation on the image fitting task, we use the FINER subset of the DIV2K dataset, randomly selecting $90\%$ of the pixels of each image for training and using the remaining $10\%$ for testing. Training is performed with Mean Square Error (MSE) loss, a batch size of $65{,}536$ pixels, and evaluation is based on Peak Signal-to-Noise Ratio (PSNR). All experiments use sinusoidal MLPs with $\omega_0 = 30$ trained for $5000$ epochs.  
For \method{}, we first train for $250$ epochs with MSE to capture low-frequency information, then add $128$ new input neurons and fine-tune for $2000$ epochs. Next, we train for $2250$ epochs with TWD, prune $384$ input neurons, and fine-tune the resulting network for an additional $500$ epochs. The resulting small network of size $[256, 512, 512]$ is compared against a model of the same size trained from scratch with MSE for $5000$ epochs.  

Table~\ref{tab: exp_surfaces} (right) shows that \method{} applied to SIREN and FINER networks achieves better convergence than standard training applied directly to either large or small networks. \method{} improves mean accuracy by $2.47$ dB on SIREN and $1.04$ dB on FINER, consistently outperforming standard MSE training. These results demonstrate that \method{} effectively transfers information from the overparameterized model to its small counterpart.



\subsection{Ablations}
\label{chap: experiments_pruning}

\paragraph{Effect of varying pruning rate.} We now ablate key design choices of \method{}, focusing on the rate of pruning and densification during training. 
First, we analyze the effect of pruning on reconstruction quality. 
We compare the accuracy drop of an adapted INR relative to a pre-trained network of size $[512, 512, 512]$ (528K parameters), which achieves a PSNR of $43.67$ dB (gray point in Figure~\ref{fig:prun_evol_cols}, right). 
We apply \method{} to the same architecture, starting with standard training for $2250$ epochs, followed by selecting $p\%$ of neurons from each hidden layer to prune ($p \in \{0.2, 0.4, 0.6, 0.8\}$) and training with TWD for another $2250$ epochs. 
Finally, we prune the selected neurons and fine-tune the resulting network for $500$ epochs, totaling $5000$ epochs of adaptation.



\begin{figure}[h]
    \hfill
    \begin{subfigure}{0.75\columnwidth}%
         \centering
         \includegraphics[width=\columnwidth]{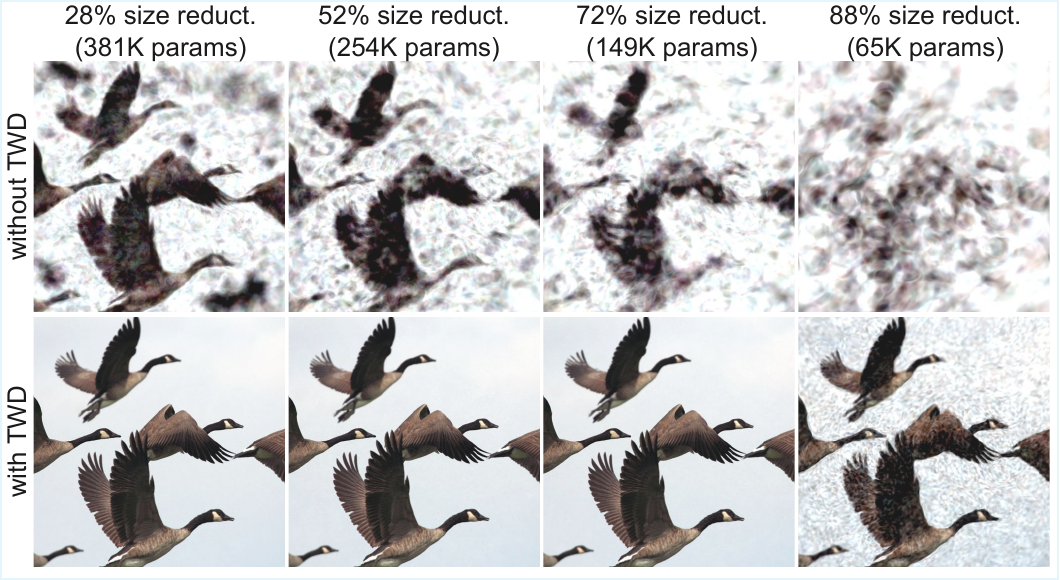}
     \end{subfigure}%
     \hfill%
     \begin{subfigure}{0.25\columnwidth}%
        \centering
        \setlength{\tabcolsep}{2pt}%
        \begin{tabular}{ll|ll}
        \hline
          &
          \textbf{\begin{tabular}[c]{@{}l@{}}Size\\ reduct.\end{tabular}} &
          \textbf{\begin{tabular}[c]{@{}l@{}}w/o \\{\scriptsize TWD}\end{tabular}} &
          \textbf{\begin{tabular}[c]{@{}l@{}}w/ \\{\scriptsize TWD}\end{tabular}}
          \\ \hline
        \multirow{4}{*}{\rotatebox[origin=c]{90}{\textbf{w/o FT}}} & 28\% & 14.35 & 43.43 \\
         & 48\% & 15.61 & 40.28 \\
         & 72\% & 14.07 & 41.25 \\
         & 88\% & 14.00 & 21.60 \\[.15cm]
        \multirow{4}{*}{\rotatebox[origin=c]{90}{\textbf{w/ FT}}} & 28\% & 41.24 & 43.57 \\
         & 48\% & 42.93 & 43.54 \\
         & 72\% & 41.98 & 42.75 \\
         & 88\% & 37.28 & 38.25 \\
        \hline
        \end{tabular}
        \vspace{0.5em} 
     \end{subfigure}%
     \hfill
\caption{\textbf{TWD reduces the dependence of finetune (FT) when pruning.} TWD effectively transfers information before pruning. 
Left: qualitative results with $28\%, 52\%, 72\%$, and $88\%$ of parameters pruned. 
The first row shows results without TWD, and the second row with TWD. 
Right: Table with the PSNR values for each case.}
\label{fig:prun_evol_cols}
\end{figure}

\begin{wraptable}[8]{r}{8.cm}
\centering
\footnotesize
\vspace{-0.4cm}
\caption{\textbf{Effect of pruning and densification on SIREN and FINER networks} (DIV2K).}
\label{tab:prune-densify-ablation}
\begin{tabular}{l|ll}
\hline
\textbf{Method} & \textbf{SIREN PSNR ↑} & \textbf{FINER PSNR ↑} \\ \hline
Small         & 36.44 ± 4.20          & 40.84 ± 3.85          \\
Prune         & 37.58 ± 3.77          & 41.80 ± 3.80          \\
Densify+Prune & \textbf{39.47 ± 4.31} & \textbf{41.88 ± 4.24} \\ \hline
\end{tabular}
\end{wraptable}
As shown in Figure~\ref{fig:prun_evol_cols}, TWD enables effective transfer of information to the remaining neurons so that a pruned network (without fine-tuning) with only $28\%$ of its weights still retains $92\%$ of the original network’s accuracy. 
In contrast, pruning without TWD leads to a severe degradation in image quality. After full training, \method{} achieves a quality drop of less than $2.1\%$ with just $28\%$ of the parameters, compared to a $3.8\%$ drop when TWD is removed from the pipeline.

\paragraph{With vs.\ without densification.} 
We ablate the role of densification in our pipeline using the same configuration as Table~2, training all models for $5000$ epochs on a subset of DIV2K. In Table~\ref{tab:prune-densify-ablation}, we compare:  
(i) a small architecture trained from scratch;  
(ii) a large model pruned (without densification) to match the small architecture; and
(iii) our proposed \method{} scheme, which iteratively adds and removes input neurons until matching the small architecture.  
Pruning alone yields a modest accuracy gain for SIREN, while the \emph{Densify+Prune} (\method{}) strategy provides a substantial boost. For FINER, pruning slightly improves reconstruction quality, but densification brings little benefit -- consistent with the fact that FINER models are already more expressive and less dependent on additional frequency capacity.

\paragraph{Pruning after densification vs. before densification.} 
Intuitively, increasing model capacity before removing redundancies should improve convergence. 
To verify this, we train a network of size $[256, 512, 256]$ and compare two schedules: pruning before densification and pruning after densification.  
For \textit{densify-then-prune}, we train for $400$ epochs with MSE, add $128$ input neurons, fine-tune for $200$ epochs, train with TWD for $200$ epochs, prune $50\%$ of the neurons in the second hidden layer, and fine-tune for $1200$ epochs.  
For \textit{prune-then-densify}, we train for $200$ epochs with MSE, continue for $200$ epochs with TWD, prune $50\%$ of the second hidden layer, fine-tune for $200$ epochs, add $128$ neurons, and finally fine-tune for $1400$ epochs.  
As a baseline, we also train a network with the final small network $[512, 256, 256]$ from scratch for $2000$ epochs.  

\begin{figure}[ht]
\centering
\begin{subfigure}[b]{\columnwidth}
    \centering
    \includegraphics[width=\columnwidth]{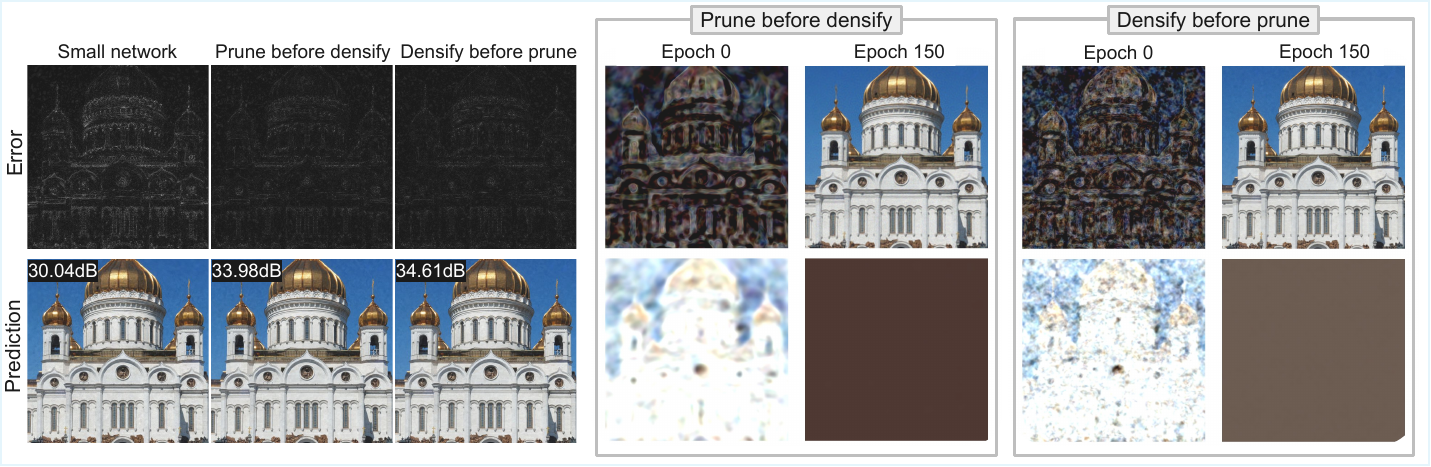}
\end{subfigure}
\caption{Comparison of training strategies with the small network, prune-before-densify, and densify-before-prune. Gray boxes show information transfer during TWD. The $1$st row displays inferences using the most contributive neurons, the $2$nd row shows reconstructions for the redundant neurons selected. Now, each column present the inferences after $t\in\{0, 150\}$ epochs of starting TWD regularization.}
\label{fig:dp-pd}
\end{figure}

The average PSNR of the small network is $30.06$ dB. 
By contrast, \textit{prune-then-densify} yields $33.41$ dB, while \textit{densify-then-prune} reaches $33.99$ dB, demonstrating that both strategies outperform standard training. 
Figure~\ref{fig:dp-pd} (left) shows that standard training produces worse error maps, while the purple boxes on the right illustrate that \textit{densify-before-prune} better preserves high-frequency details compared to \textit{prune-before-densify}.

\section{Conclusion}
\label{sec:conclusion}

We introduced \method{}, a dynamic training framework for implicit neural representations (INRs) that adaptively aligns network architecture with the complexity of the target signal. The framework integrates two complementary components: \emph{pruning}, which removes redundant neurons to mitigate overparameterization, and \emph{densification}, which expands the network’s expressivity by selectively introducing new input frequencies based on a principled spectral analysis.

Our approach contributes toward automating architecture adaptation in INR learning, offering a more efficient and flexible alternative to static design choices.
As future work, we aim to develop more advanced mechanisms for information transfer during pruning, extend our method to a broader class of architectures, and explore its applicability to more data modalities beyond images and surfaces.



\bibliographystyle{plainnat}
\bibliography{iclr2026_conference}

\appendix

\section{Proofs}

\subsection{Theorem 1}

In the main paper, we present an identity (Thrm 1) derived by \cite{novello2024taming}, which linearizes the $j$-th hidden neuron of the $(i+1)$-th layer, $h_j^{i+1}$.
Similar results have been presented in~\cite{yuce2022structured} for the case of shallow SIRENs.
The identity below extends this analysis to hidden neurons at arbitrary depths.
\begin{theorem}
    The hidden neuron $h_j^{i+1}$ admits the following amplitude-phase expansion:
    \emph{\begin{align}
        h_j^{i+1}({\x}) = \sum_{\textbf{k} \in \mathbb{Z}^{n_i}} \alpha_{\textbf{k}} \, \sin\left( \langle \textbf{k}, \textbf{y}^i \rangle + b_j^{i+1} \right),\qquad \text{where}\quad 
        |\alpha_{\textbf{k}}| \leq \prod_{l} \left( \frac{|W^{i+1}_{jl}|}{2} \right)^{|k_l|} \frac{1}{|k_l|!}.
    \end{align}}
    Here, $\alpha_\textbf{k}=\prod_{l=1}^{n_i}J_{k_l}(W_{jl}^{i+1})$ is the product of Bessel functions.
\end{theorem}

\noindent Before starting the proof, recall that we defined $h^{i+1}_j(\textbf{x}) = \sin\left(\sum_{l=1}^{n_i} W_{jl}^{i+1}\sin({y}^i_l)+b^{i+1}_j \right)$, with $\textbf{y}^i=\left[y_l^i\right]_l$ the linear, non-activated contribution of the $i$-th layer.
To simplify notation, we drop the indices $i$ and $i+1$ from $\W^{i+1}, \textbf{b}^{i+1}, \textbf{y}^i, \text{ and } n_i$.

\begin{proof}
The first part of the proof consists of verifying 
\begin{align}\label{e-perceptron_approx}
\begin{split}
&\sin\left(
    \sum_{l=1}^{n} W_{jl}\sin(
y_l
    )+ b_j
\right)
 = \sum_{\textbf{k}\in\Z^m}\alpha_\textbf{k}\sin\big(
    \dotprod{\textbf{k}, \textbf{y}}+ b_j
\big)  \quad \text{and}
\\
&\cos\left(\!
    \sum_{l=1}^{n} W_{jl}\sin(
y_l
    )+ b_j
\right)
= \sum_{\textbf{k}\in\Z^m}\alpha_\textbf{k}\cos\big(
    \dotprod{\textbf{k}, \textbf{y}}+ b_j
\big).
\end{split}
\end{align}

The proof is by induction in $n$. For the \textbf{base} case $n=1$,  we use the sum of angles identities and the Bessel function of the first kind properties (see~\cite[num. 9.1.42, 9.1.43]{abramowitz1964handbook}) to prove $\sin\left(W_{j1} \sin(y_1)+b_j\right)=\sum_{k\in \Z} J_k(W_{j1}) \sin(ky_1+b_j)$:
\begin{align*}
    \sin\big(W_{j1} \sin(y_1)+b_j\big)&=\sin\big(W_{j1} \sin(y_1)\big)\cos(b_j)+\cos\big(W_{j1} \sin(y_1)\big)\sin(b_j)\\
    &=\sum_{k\in\Z \text{ odd}} \!\!\!J_k(W_{1j}) \sin(ky_1)\cos(b_j) + \sum_{l\in\Z \text{ even}} \!\!\!J_l(W_{1j}) \cos(ly_1)\sin(b_j)\\
    &=\sum_{k\in\Z \text{ odd}} \!\!\!J_k(W_{j1}) \sin(ky_1+b_j) + \sum_{l\in\Z \text{ even}} \!\!\!J_l(W_{j1}) \sin(ly_1+b_j)\\
    &=\sum_{k\in\Z} J_k(W_{j1}) \sin(ky_1+b_j).
\end{align*}
In the third equality we combined the formula $\sin(u)\cos(v)=\frac{\sin(u+v)+\sin(u-v)}{2}$ and the fact that $J_{-k}(u)=(-1)^kJ_k(u)$ to rewrite the summations.
The proof of the cosine analogous expansion $\cos\big(W_{j1}\sin(y_1)+b_j\big)\!=\!\!\sum J_l(W_{j1}+b_j) \cos(ly_1)$ is similar.

Assume that \eqref{e-perceptron_approx} hold for $n-1$, with $n>1$, we prove that it also holds for $n$ (the \textbf{induction~step}). 
\begin{align*}
    \sin\left(\sum_{l=1}^{n} W_{jl}\sin(y_l)+ b_j\right)&=\sin\left(\sum_{l=1}^{n-1} W_{jl}\sin(y_l)+ b_j\right)\cos\big(W_{jn}\sin(y_n)\big)\\
    &+\cos\left(\sum_{l=1}^{n-1} W_{jl}\sin(y_l)+ b_j\right)\sin\big(W_{jn}\sin(y_n)\big)
    \\
    &=\sum_{\textbf{k}\in\Z^{n-1},\, l\in\Z \text{ even}}\alpha_\textbf{k}J_l(W_{jn})\sin\Big(\dotprod{\textbf{k}, \textbf{y}}+ b_j\Big)\cos(ly_n)\\
    &+\sum_{\textbf{k}\in\Z^{n-1},\, l\in\Z \text{ odd}}\alpha_\textbf{k}J_l(W_{jn})\cos\Big(\dotprod{\textbf{k}, \textbf{y}}+ b_j\Big)\sin(ly_n)\\
    &=\sum_{\textbf{k}\in\Z^n}\alpha_\textbf{k} \sin\Big(\dotprod{\textbf{k}, \textbf{y}}+ b_j\Big)
\end{align*}
We use the induction hypothesis in the second equality and an argument similar to the one used in the base case to rewrite the harmonic sum. Again, the cosine activation function case is analogous.

For the second part of the proof, we must prove the inequality in Equation~\eqref{t-expansion}. For that, note that $\alpha_\textbf{k}=\prod_{l=1}^nJ_{k_l}(W_{jl})$, and that 
\begin{align}
\label{e-bessel_inequality}
    |J_k(W_{jl})|<\frac{\left( \tfrac{|W_{jl}|}{2} \right)^k}{k!}, \quad k> 0,\quad W_{jl}>0
\end{align}
But this also holds for $W_{jl}\leq0$ since $|J_k(-u)|=|J_k(u)|$, and for $k\leq0$ as $|J_{-k}(u)|=|(-1)^kJ_k(u)|=|J_k(u)|$. Then, substituting \eqref{e-bessel_inequality} in $\alpha_\textbf{k}=\prod_{l=1}^nJ_{k_l}(W_{jl})$, we obtain the desired result.
\end{proof}

\subsection{Theorem 2}
\begin{theorem}
Let $f_\theta$ be a sinusoidal INR of depth $d$, and let $\widetilde{f_\theta}$ be the network obtained by perturbing the $k$-th hidden layer weights and biases to $\widetilde{\mathbf{W}}^k$ and $\widetilde{\mathbf{b}}^k$. Then,
\vspace{-0.2cm}
\begin{align*}
    \sup_{\mathbf{x}} \left\| f_\theta(\mathbf{x}) - \widetilde{f_\theta}(\mathbf{x}) \right\|_\infty
    \leq \left( \|\mathbf{W}^k - \widetilde{\mathbf{W}}^k\|_\infty + \|\mathbf{b}^k - \widetilde{\mathbf{b}}^k\|_\infty \right) \|\mathbf{L}\|_\infty \prod_{i=k+1}^d \|\mathbf{W}^i\|_\infty.
\end{align*}
\vspace{-0.2cm}
\end{theorem}

\begin{proof}
    First, note that $\mathbf{x} \mapsto \mathbf{L} \mathbf{x}$ is $\|\mathbf{L}\|_\infty$-Lipschitz for infinity norms:
    \[ \|\mathbf{L} \mathbf{x} - \mathbf{L} \mathbf{x}'\|_\infty = \| \mathbf{L} (\mathbf{x} - \mathbf{x}') \|_\infty \leq \|\mathbf{L}\|_\infty \|\mathbf{x} - \mathbf{x}'\|_\infty. \]
    Second, note that $\mathbf{x} \mapsto \sin(\mathbf{x})$ is 1-Lipschitz also for infinity norms, and thus $\mathbf{x} \mapsto \mathbf{S}^i(\mathbf{x}) = \sin(\mathbf{W}^i \mathbf{x} + \mathbf{b}^i)$ is also $\|\mathbf{W}^i\|_\infty$-Lipschitz:
    \begin{align*}
        \|\sin(\mathbf{W}^i \mathbf{x} + \mathbf{b}^i) - \sin(\mathbf{W}^i \mathbf{x}' + \mathbf{b}^i)\|_\infty
        &\leq \|\sin\|_\mathrm{Lip} \|(\mathbf{W}^i \mathbf{x}' + \mathbf{b}^i) - (\mathbf{W}^i \mathbf{x}' + \mathbf{b}^i)\|_\infty
        \\ &\leq \|\mathbf{W}^i (\mathbf{x} - \mathbf{x}')\|_\infty
        \leq \|\mathbf{W}^i\|_\infty \|\mathbf{x} - \mathbf{x'}\|_\infty.
    \end{align*}
    It thus follows that, for any $\mathbf{x}$:
    \begin{align*}
        & \left\| (\mathbf{L} \circ \mathbf{S}^d \circ \cdots \circ \mathbf{S}^k \circ \cdots \circ \mathbf{S}^0)(\mathbf{x}) - (\mathbf{L} \circ \mathbf{S}^d \circ \cdots \circ \widetilde{\mathbf{S}}^k \circ \cdots \circ \mathbf{S}^0)(\mathbf{x}) \right\|_\infty
        \\ &\leq \|\mathbf{L}\|_\infty \left\| (\mathbf{S}^d \circ \cdots \circ \mathbf{S}^k \circ \cdots \circ \mathbf{S}^0)(x) - (\mathbf{S}^d \circ \cdots \circ \widetilde{\mathbf{S}}^k \circ \cdots \circ \mathbf{S}^0)(\mathbf{x}) \right\|_\infty
        \\ &\leq \|\mathbf{L}\|_\infty \left( \prod_{i=k+1}^d \|\mathbf{W}^i\|_\infty \right) \left\| (\mathbf{S}^k \circ \cdots \circ \mathbf{S}^0)(\mathbf{x}) - (\widetilde{\mathbf{S}}^k \circ \cdots \circ \mathbf{S}^0)(\mathbf{x}) \right\|_\infty
        \\ &= \|\mathbf{L}\|_\infty \left( \prod_{i=k+1}^d \|\mathbf{W}^i\|_\infty \right) \Bigl\| \sin(\mathbf{W}^k (\mathbf{S}^{k-1} \circ \cdots \circ \mathbf{S}^0)(\mathbf{x}) + \mathbf{b}^k) \\ &\qquad\qquad\qquad\qquad\qquad\qquad - \sin(\widetilde{\mathbf{W}}^k (\mathbf{S}^{k-1} \circ \cdots \circ \mathbf{S}^0)(\mathbf{x}) + \widetilde{\mathbf{b}}^k) \Bigr\|_\infty
        \\ &\leq \|\mathbf{L}\|_\infty \left( \prod_{i=k+1}^d \|\mathbf{W}^i\|_\infty \right) \|\sin\|_\mathrm{Lip} \Bigl\| (\mathbf{W}^k (\mathbf{S}^{k-1} \circ \cdots \circ \mathbf{S}^0)(\mathbf{x}) + \mathbf{b}^k) \\ &\qquad\qquad\qquad\qquad\qquad\qquad\qquad\qquad - (\widetilde{\mathbf{W}}^k (\mathbf{S}^{k-1} \circ \cdots \circ \mathbf{S}^0)(\mathbf{x}) + \widetilde{\mathbf{b}}^k) \Bigr\|_\infty
        \\ &= \|\mathbf{L}\|_\infty \left( \prod_{i=k+1}^d \|\mathbf{W}^i\|_\infty \right) \left\| (\mathbf{W}^k - \widetilde{\mathbf{W}}^k) (\mathbf{S}^{k-1} \circ \cdots \circ \mathbf{S}^0)(\mathbf{x}) + (\mathbf{b}^k - \widetilde{\mathbf{b}}^k) \right\|_\infty
        \\ &\leq \|\mathbf{L}\|_\infty \left( \prod_{i=k+1}^d \|\mathbf{W}^i\|_\infty \right) \left( \left\| (\mathbf{W}^k - \widetilde{\mathbf{W}}^k) (\mathbf{S}^{k-1} \circ \cdots \circ \mathbf{S}^0)(\mathbf{x}) \right\|_\infty + \left\| \mathbf{b}^k - \widetilde{\mathbf{b}}^k \right\|_\infty \right)
        \\ &\leq \|\mathbf{L}\|_\infty \left( \prod_{i=k+1}^d \|\mathbf{W}^i\|_\infty \right) \left( \|\mathbf{W}^k - \widetilde{\mathbf{W}}^k\|_\infty \left\| (\mathbf{S}^{k-1} \circ \cdots \circ \mathbf{S}^0)(\mathbf{x}) \right\|_\infty + \left\| \mathbf{b}^k - \widetilde{\mathbf{b}}^k \right\|_\infty \right).
    \end{align*}
    Finally, note that since sines lie in $[-1, +1]$, it must hold that $\left\| (\mathbf{S}^{k-1} \circ \cdots \circ \mathbf{S}^0)(\mathbf{x}) \right\|_\infty = \max_i \left| [(\mathbf{S}^{k-1} \circ \cdots \circ \mathbf{S}^0)(\mathbf{x})]_i  \right| \leq \max_i 1 = 1$, from which we conclude the proof.
\end{proof}

\section{Signed distance functions}

\begin{table}[]
\centering
\caption{Quantitative comparisons on representing surfaces from the Stanford 3D Scanning Repository with $\omega_0=30$. We compare the adapted network trained with \textbackslash{}method\{\}, the large network with standard training (large), and the model with small architecture and standard training (small). We report the average chamfer distance (Avg CD ($\times10^2$)) between reconstructed and ground-truth surfaces and the percentage of network parameters compared to the large architecture (lower is better).}
\label{tab:exp_surfaces_w0_30}
\begin{tabular}{llll}
\hline
\textbf{\begin{tabular}[c]{@{}l@{}}Model\\ \textsuperscript{(SDFs)}\end{tabular}} &
  \textbf{Variant} &
  \textbf{CD (×10²) ↓} &
  \textbf{\begin{tabular}[c]{@{}l@{}}Size \\ reduct. ↓\end{tabular}} \\ \hline
\multirow{3}{*}{SIREN} & Large & \textbf{0.56 ± 0.08} & -     \\
                       & Small & 0.59 ± 0.09          & 62.14 \\
                       & Ours  & 0.58 ± 0.06          & 62.14 \\[1.5mm]
\multirow{3}{*}{FINER} & Large & \textbf{0.63 ± 0.09} & -     \\
                       & Small & 0.67 ± 0.09          & 62.14 \\
                       & Ours  & \textbf{0.63 ± 0.06} & 62.14 \\ \hline
\end{tabular}
\end{table}

In Table~\ref{tab:exp_surfaces_w0_30} we evaluate AIRe (‘Ours’) against an overparametrized, large INR of size [256, 256, 256] and a small, reduced model of size [128, 128, 256]. These last two are fitted with the standard training pipeline, and the reconstruction quality was measured using the Chamfer Distance ($\times10^2$).

Table~\ref{tab: exp_surfaces_breakdown} shows a per-scene breakdown of the SDF quantitative results presented in the main paper when $\omega_0=60$ and and the small network size is $[64, 64, 256]$.
The per-scene breakdown is consistent with the aggregate quantitative metrics.
Our method outperforms the small network in all cases.
It also obtains similar or better accuracy compared to the large network but uses roughly 1/6 of network parameters.

\begin{table}[h!]
\centering
\caption{Per-scene quantitative comparisons on representing surfaces from the Stanford 3D Scanning Repository with $\omega_0=60$ and model size $[64, 64, 256]$. We compare \method{}, the network with large architecture, and the model with small architecture. We report the chamfer distance (CD ($\times10^2$)) between reconstructed and ground-truth surfaces (lower is better). Best values in \textbf{bold}, second best values \underline{underlined}.}
\label{tab: exp_surfaces_breakdown}
\small
\begin{tabular}{@{}ll|rrrrr@{}}
\toprule
\multirow{2}{*}{\textbf{Model}} & \multirow{2}{*}{\textbf{Variant}} & \multicolumn{5}{c}{\textbf{CD} $\mathbf{(\times10^2) \downarrow}$}                                                                                                     \\ \cmidrule(l){3-7} 
                              &                         & \multicolumn{1}{l}{\textbf{Armadillo}} & \multicolumn{1}{l}{\textbf{Bunny}} & \multicolumn{1}{l}{\textbf{Dragon}} & \multicolumn{1}{l}{\textbf{Happy Buddha}} & \multicolumn{1}{l}{\textbf{Lucy}} \\ \midrule
\multirow{3}{*}{SIREN}        & Large                 & \textbf{0.60}                 & \underline{0.75}                & \underline{0.65}                 & \textbf{0.50}                    & \underline{0.74}               \\
                              & Small                   & 0.99                          & 0.79                      & 0.82                       & 0.98                             & 0.86                     \\
                              & Ours          & \underline{0.65}                    & \textbf{0.69}             & \textbf{0.62}              & \underline{0.64}                       & \textbf{0.61}            \\ \midrule
\multirow{3}{*}{FINER}        & Large                 & \underline{2.13}                    & \underline{2.06}                & \underline{2.17}                 & 2.74                             & \underline{1.60}               \\
                              & Small                   & 5.51                          & 10.80                     & 4.57                       & \underline{2.58}                       & 1.92                     \\
                              & Ours          & \textbf{0.88}                 & \textbf{0.95}             & \textbf{0.73}              & \textbf{1.10}                    & \textbf{0.76}            \\ \bottomrule
\end{tabular}
\end{table}

Figure~\ref{fig:exp_surfaces_additional} shows additional examples of surface reconstructions using SIREN with $\omega_0=60$ and model architecture $[64,64, 256]$.
As in the other examples, the surface trained using our method presented a lower error compared to the small network.
We also see in Figure~\ref{fig: exp_surfaces_finer} an example using FINER with settings $\omega_0=60$ and network architecture $[64,64, 256]$.
Note that \method{} offers a better reconstruction than the small network with less artifacts.
\begin{figure}[h!]
    \centering
    \includegraphics[width=\linewidth]{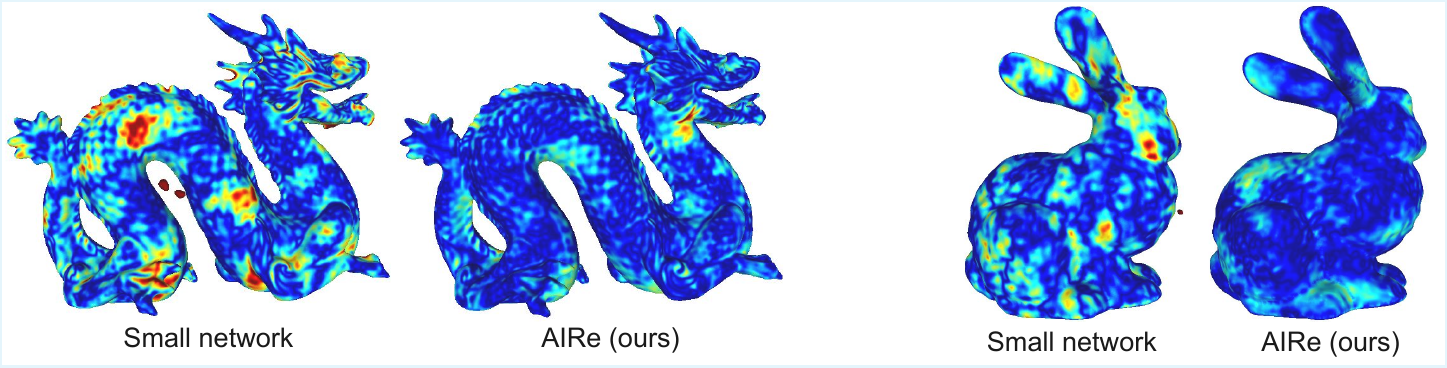}     
    \caption{Additional qualitative comparisons on representing surfaces based on the Dragon and Bunny models using a SIREN network with $\omega_0=60$ and model architecture $[64,64, 256]$.
    Left: results of training the small network.
    Right: results of \method{}.
    We illustrate the unsigned distance from the ground-truth surface using a color scale from dark blue (zero) to dark red ($\ge0.01$).}
    \label{fig:exp_surfaces_additional}
\end{figure}

\begin{figure}[h!]
    \centering
    \includegraphics[width=0.6\linewidth]{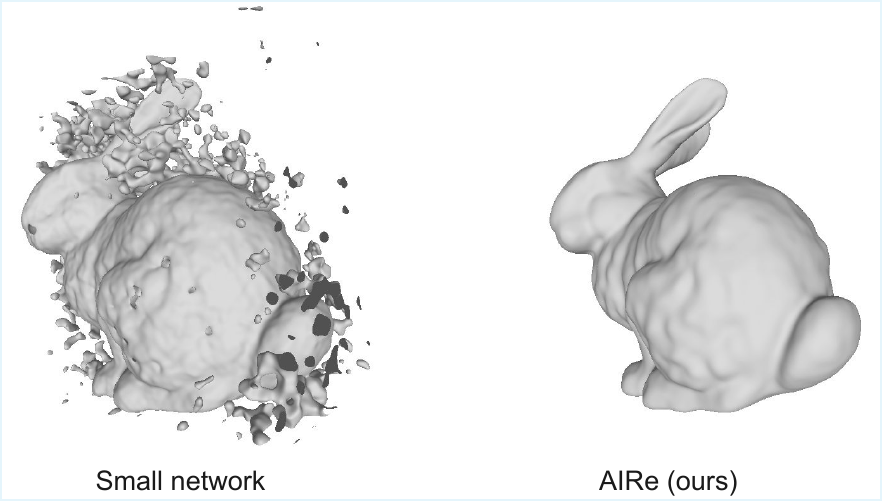}     
    \caption{A case with the Bunny model using a FINER network with $\omega_0=60$ and model size $[64, 64, 256]$. The final, small network (left) presents several artifacts, which does not occur with the \method{} method (right).}
    \label{fig: exp_surfaces_finer}
\end{figure}

Finally, we present Table~\ref{tab:surfaces_times}, where we compare the time overhead when training the large, small and adapted models,. Observe that the \method{} and the large model have an equivalent training time but with only $16.04\%$ of the original parameters. Furthermore, observe that a small model trained during the same amount of time has much worse accuracy than a network trained with \method{}.

\begin{table}[!htb]
\centering
\caption{\textbf{Time overhead comparisons when training SDFs.} We consider a \textit{Large} network with $\sim133K$ parameters and a \textit{Small} network with $\sim 22K$ parameters that are fitted for 1000 epochs. We compare them with a network adapted during 1000 epochs using AIRe up to a size equal to the small network. We also train the final network for 2000 epochs to compare the reconstruction quality along a time budget.}
\label{tab:surfaces_times}
\footnotesize
\setlength{\tabcolsep}{3pt}

\begin{tabular}{l|lcl}
\hline
\multicolumn{1}{c|}{\textbf{\begin{tabular}[c]{@{}c@{}}Variant\end{tabular}}} &
  \multicolumn{1}{c}{\textbf{\begin{tabular}[c]{@{}c@{}}CD \\ ($\times10^2$) ↓\end{tabular}}} &
  \multicolumn{1}{c}{\textbf{\begin{tabular}[c]{@{}c@{}}Size \\ reduct. ↓\end{tabular}}} &
  \multicolumn{1}{c}{\textbf{\begin{tabular}[c]{@{}c@{}}Time \\ (s) ↓\end{tabular}}} \\ \hline
Large                     & 0.65 & - & 76.0 \\
Small ($10^3$ ep)        & 0.89 & 16.04$\%$  & 39.4 \\
Small (2$\times10^3$ ep) & 0.86 & 16.04$\%$  & 83.2 \\
Ours                      & 0.64 & 16.04$\%$  & 76.8 \\ \hline
\end{tabular}
\hspace{0.1cm}
\begin{tabular}{l|lccl}
\hline
\multicolumn{1}{c|}{\textbf{SDF}} &
  \multicolumn{1}{c}{\textbf{Large}} &
  \multicolumn{1}{c}{\textbf{\begin{tabular}[c]{@{}c@{}}Small \\ ($10^3$ ep)\end{tabular}}} &
  \multicolumn{1}{c}{\textbf{\begin{tabular}[c]{@{}c@{}}Small\\ ($2\!\times\!10^3$ ep)\end{tabular}}} &
  \multicolumn{1}{c}{\textbf{\begin{tabular}[c]{@{}c@{}}Ours\end{tabular}}} \\ \hline
Armadillo    & 52   & 27   & 55   & 52   \\
Bunny        & 34   & 17   & 35   & 34   \\
Dragon       & 60   & 31   & 65   & 60   \\
Happy Buddha & 159  & 83   & 179  & 162  \\
Lucy         & 75   & 39   & 82   & 76   \\ \hline
Avg. time (s)      & 76.0 & 39.4 & 83.2 & 76.8 \\ \hline
\end{tabular}
\end{table}

\section{Images}
We present ablation studies to support the choice of hyperparameters for \method{}.
First, we investigate the optimal allocation of epochs between the targeted weight decay stage and the fine-tuning stage under a fixed training budget. Specifically, we train SIREN~\cite{sitzmann2020implicit} and FINER~\cite{liu2024finer} models, each with two hidden layers of 512 neurons, for a total of 5000 epochs.
Training begins with standard optimization for $x$ epochs, followed by targeted weight decay for $y$ epochs, where $x, y \in \{100, 750, 1000, 1250, 1500, 1750, 2000, 2250\}$. The remaining $5000 - x - y$ epochs are allocated to fine-tuning.

\begin{figure}[h!]
    \centering
    \includegraphics[width=0.47\linewidth]{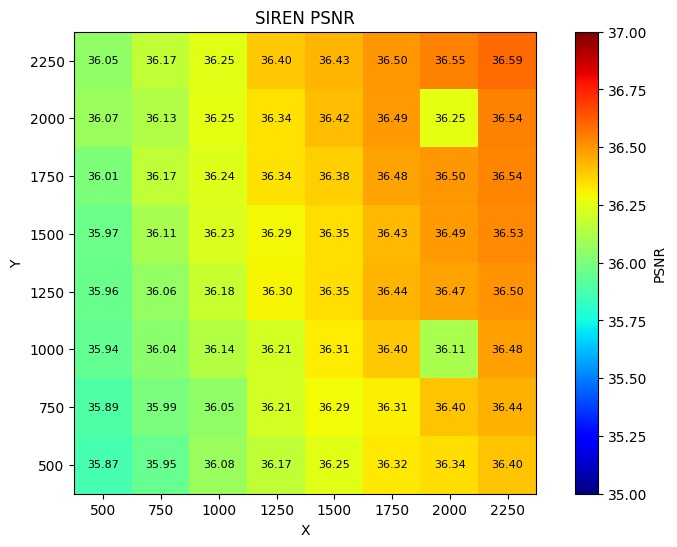}
    \includegraphics[width=0.47\linewidth]{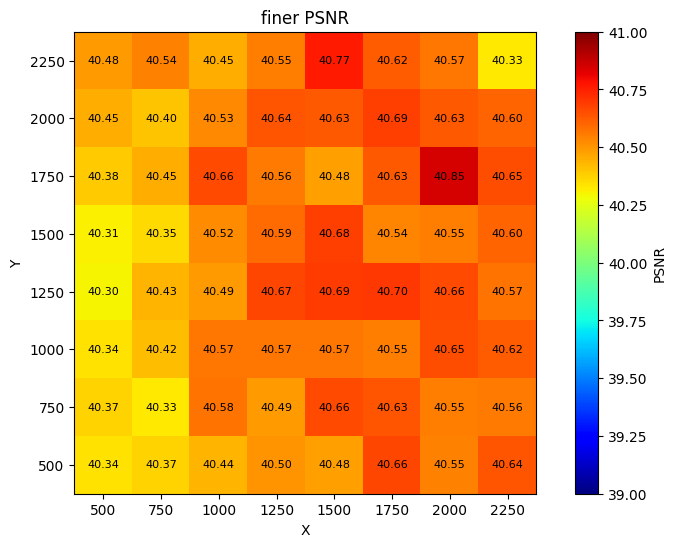}
    \caption{Ablation of the number of epochs used to pre-train the model (x axis), to train with targeted weight decay (y axis) and fine tune.
    The total training lasts for 5000 epochs, and each value refers to the mean PSNR over the DIV2K dataset. (Left) SIREN~\cite{sitzmann2020implicit} architecture shows that longer standard training and targeted weight decay stage improve quality, even with fewer fine tuning epochs. (Right) FINER architecture shows less consistency in the results, demonstrating that above $1000$ epochs of standard training the results improve, but show no clear pattern.}
    \label{fig: imgs_epochs}
\end{figure}

Figure~\ref{fig: imgs_epochs} shows the PSNR for each epoch distribution, where the $x$-axis corresponds to the number of epochs used for the initial standard training stage, and the $y$-axis indicates the number of epochs allocated to the targeted weight decay stage.
As shown, SIREN models benefit from increased training time in both the standard training and targeted weight decay stages, resulting in improved reconstruction accuracy. In contrast, FINER models show only marginal improvements when the initial training stage exceeds 1000 epochs.

To determine the optimal pruning configuration, both in terms of which layers to prune and the amount per layer, we train models with the best-performing epoch distribution for both SIREN and FINER over 5000 epochs, applying varying levels of pruning to each layer.
Figure~\ref{fig: imgs_perc_prune} presents the PSNR of each reconstruction, where $x$ represents the percentage of neurons pruned in the first layer, and $y$ denotes the percentage pruned in the second layer.

\begin{figure}[h!]
    \centering
    \includegraphics[width=0.47\linewidth]{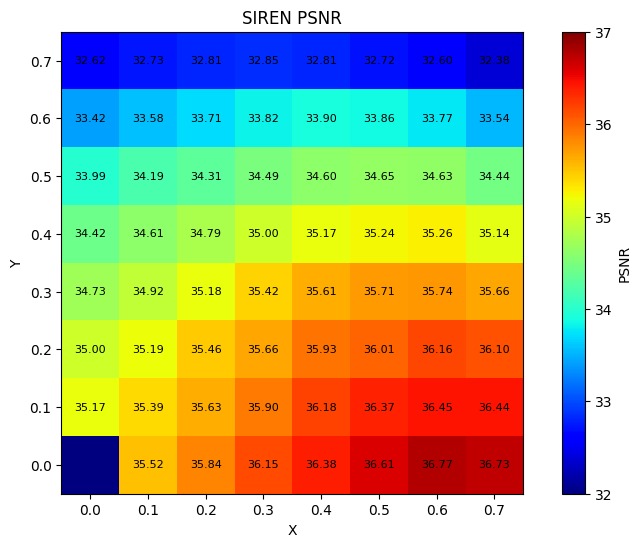}
    \includegraphics[width=0.47\linewidth]{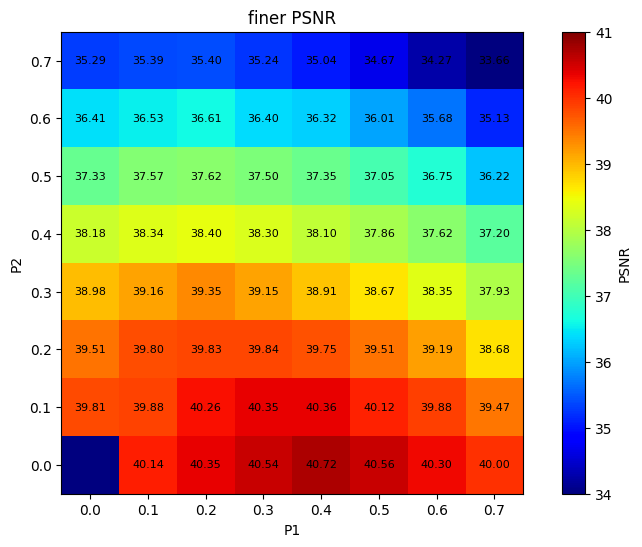}
    \caption{Ablation of the 
    quality degradation with respect to the percentage of prune on the 1st hidden layer (x axis) and the 2nd hidden layer (y axis).
    The total training lasts for 5000 epochs, and each value refers to the mean PSNR over the DIV2K dataset. Observe that both images show that pruning the 1st layer retains more quality than pruning the 2nd layer. (Left) A SIREN~\cite{sitzmann2020implicit} architecture reconstruction quality is preserved even with an extreme prune of $60\%$. (Right) FINER architecture quality is better retained when pruning the first layer, albeit with less percentage.}
    \label{fig: imgs_perc_prune}
\end{figure}

Both SIREN and FINER benefit from pruning the first layer, although the optimal percentage of neuron removal differs between the two. In contrast, pruning hidden layers generally leads to a degradation in reconstruction quality.

We also perform an ablation study on the use of regularization to improve neuron removal during training. Specifically, we train a sinusoidal INR using three configurations: standard weight decay, targeted weight decay, and no regularization prior to pruning. Standard weight decay yields the lowest reconstruction accuracy at 34.2dB, while removing regularization improves the result by 0.51dB. The targeted weight decay stage achieves the best performance, increasing accuracy to 36.9~dB.

For the densification strategy, we examine the impact of varying both the number of training epochs before densification and the percentage of new input neurons added. Concretely, the INR is initially trained for $x$ epochs, then its first layer is expanded by $(y*100)\%$, and the augmented network is fine-tuned for the remaining $3000 - x$ epochs.
The results are presented in Figure~\ref{fig: imgs_densification}.

\begin{figure}[h!]
    \centering
    \includegraphics[width=0.6\linewidth]{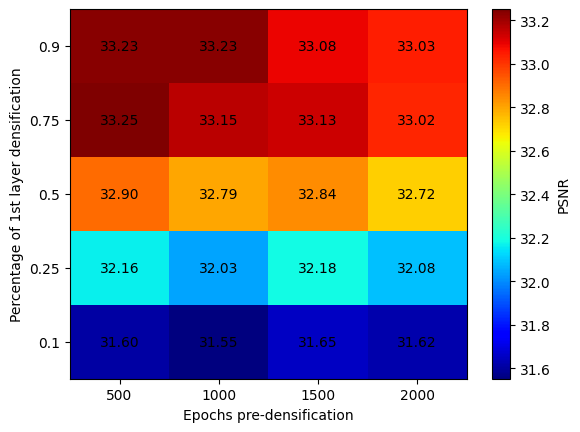}
    \caption{Ablation on the epochs trained before densifying ($x$ axis) compared to the percentage of input neurons added ($y$ axis).}
    \label{fig: imgs_densification}
\end{figure}

As expected, increasing the number of neurons leads to higher PSNR values. Additionally, accuracy improves when a larger number of neurons is added early in the training process (i.e., before 1500 epochs).

\section{Additional discussions}

We provide further details regarding the parameter settings used in our experiments.
The targeted weight decay stage is trained using the following loss function $$\mathcal{L}_{\alpha, \mathcal{I}} = \mathcal{L}_{\text{data}} + \alpha\sum_{j\in\mathcal{I}}\|\W_{*j}\|_1,$$ where $\alpha$ is a parameter that starts at zero and increases linearly up to one at the end of this stage.
For the pruning scheme, we use the Prune package in PyTorch, using structured masks over the to remove the corresponding weights.
Specifically, when pruning neuron $h^i_j$, we mask the entries of the $j$-th column of $\W^i$.
This removes all the neuron's influence from the network.
As for the densification technique, we preserve the optimizer state of previous neurons to minimize training affectation.

We trained using a 12 GB NVIDIA GPU (TITAN X Pascal) and a 24 GB NVIDIA GPU (RTX 4090).

\subsection{Neural Radiance Fields}

\begin{figure*}
     \captionsetup[subfigure]{labelformat=empty}
     \centering
     \begin{subfigure}[b]{0.325\textwidth}
         \centering
         \includegraphics[width=\textwidth]{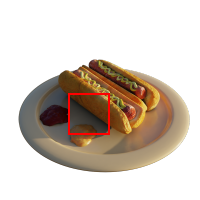}
     \end{subfigure}
     \begin{subfigure}[b]{0.325\textwidth}
         \centering
         \includegraphics[width=\textwidth]{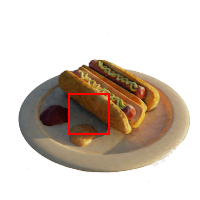}
     \end{subfigure}
     \begin{subfigure}[b]{0.325\textwidth}
         \centering
         \includegraphics[width=\textwidth]{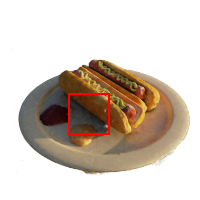}
     \end{subfigure}
     \begin{subfigure}[b]{0.325\textwidth}
         \centering
         \includegraphics[width=0.5\textwidth]{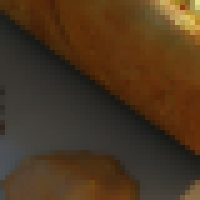}
     \end{subfigure}
     \begin{subfigure}[b]{0.325\textwidth}
         \centering
         \includegraphics[width=0.5\textwidth]{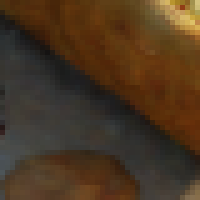}
     \end{subfigure}
     \begin{subfigure}[b]{0.325\textwidth}
         \centering
         \includegraphics[width=0.5\textwidth]{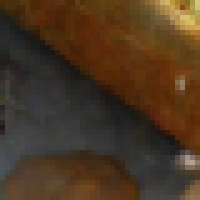}
     \end{subfigure}
     \begin{subfigure}[b]{0.325\textwidth}
         \centering
         \includegraphics[width=\textwidth]{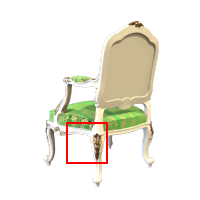}
     \end{subfigure}
     \begin{subfigure}[b]{0.325\textwidth}
         \centering
         \includegraphics[width=\textwidth]{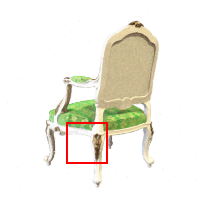}
     \end{subfigure}
     \begin{subfigure}[b]{0.325\textwidth}
         \centering
         \includegraphics[width=\textwidth]{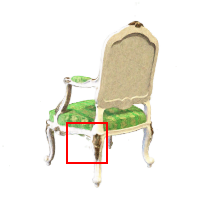}
     \end{subfigure}
     \begin{subfigure}[b]{0.325\textwidth}
         \centering
         \includegraphics[width=0.5\textwidth]{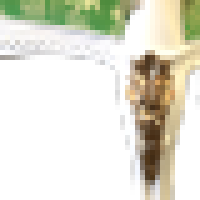}
     \end{subfigure}
     \begin{subfigure}[b]{0.325\textwidth}
         \centering
         \includegraphics[width=0.5\textwidth]{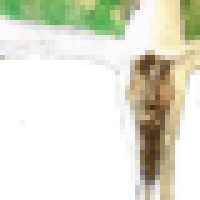}
     \end{subfigure}
     \begin{subfigure}[b]{0.325\textwidth}
         \centering
         \includegraphics[width=0.5\textwidth]{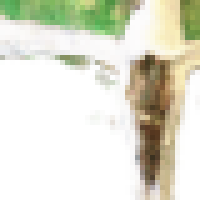}
     \end{subfigure}
     \begin{subfigure}[b]{0.325\textwidth}
         \centering
         \includegraphics[width=\textwidth]{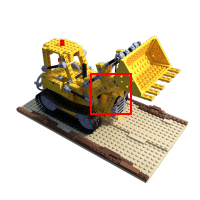}
     \end{subfigure}
     \begin{subfigure}[b]{0.325\textwidth}
         \centering
         \includegraphics[width=\textwidth]{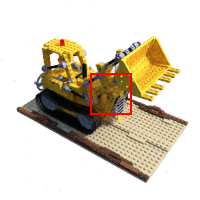}
     \end{subfigure}
     \begin{subfigure}[b]{0.325\textwidth}
         \centering
         \includegraphics[width=\textwidth]{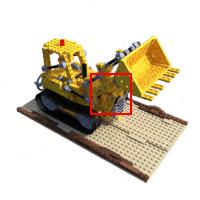}
     \end{subfigure}
     \begin{subfigure}[b]{0.325\textwidth}
         \centering
         \includegraphics[width=0.5\textwidth]{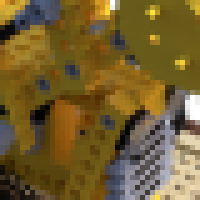}
         \caption{Ground Truth}
     \end{subfigure}
     \begin{subfigure}[b]{0.325\textwidth}
         \centering
         \includegraphics[width=0.5\textwidth]{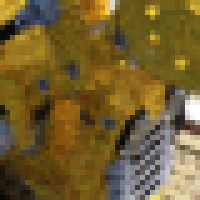}
         \caption{TWD+P (Ours)}         
     \end{subfigure}
     \begin{subfigure}[b]{0.325\textwidth}
         \centering
         \includegraphics[width=0.5\textwidth]{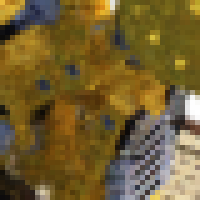}
         \caption{Small}
     \end{subfigure}
     
        \caption{Qualitative comparisons on representing NeRFs using the Hotdog (top), Chair (middle), and Lego (bottom) scenes from the Blender dataset. First column: ground-truth views. Second column: results of our proposed approach of targeted weight decay for pruning (TWD+P). Third column: results of training from scratch a small network (Small) with an architecture equivalent to ours after pruning. Differences in quality highlighted by insets.}
        \label{fig: exp_nerfs}
\end{figure*}

We adopt the torch-ngp framework~\footnote{https://github.com/ashawkey/torch-ngp} for NeRF implemented by FINER that considers two networks:
A density network that takes a 3D position as input and outputs a density value and a geometric feature vector $v\in\R^{182}$; and a color network that receives $v$ and a 3D direction and returns a RGB color.
NeRF computes the color of a pixel with volume rendering using 3D points sampled on a ray traced from the center of the virtual camera through the pixel~\cite{mildenhall2021nerf}.
We set the batch size to $4,096$ rays and Adam optimizer with learning rate of $0.0002$, $\beta_1=0.9$, $\beta_1=0.99$, $\epsilon=10^{-15}$, and exponential learning rate decay of $0.1$.
We update the model weights using an exponential moving average with a decay of $0.95$.
We follow FINER's experimental setting, where for each scene of the Blender dataset, we have $25$ images for training, $200$ images for testing,  all downsampled to $200\times200$ pixels.
We employ the PSNR and the number of network parameters (Params) as evaluation~metrics.

We evaluate three approaches for NeRF training:  \textbf{Large} and \textbf{Small} networks, which train from scratch for $1.5\times10^3$ epochs a density network of size $[182, 182, 182]$ and color networks with architecture $[182, 182, 182]$ and $[91, 91, 182]$, respectively. 
On the other hand, \textbf{\method{}} considers training from scratch for $300$ epochs the same networks from Large model, then selecting $50\%$ of neurons from both hidden layers of the density network for $750$ epochs of TWD, followed by pruning of selected neurons, and finally $450$ epochs of fine-tuning of both density and color networks.
\vspace{-0.3cm}

\begin{table}[h!]
\small
\setlength{\tabcolsep}{3pt}
\centering
\caption{Quantitative comparisons between \method{}'s pruning scheme with training from scratch the `Large' and `Small' networks. We report the average PSNR between reconstructed and ground-truth test views (higher is better), the PSNR difference with respect to Large (higher is better), and the percentage of network parameters with respect to Large (higher is better). Best values in \textbf{bold}, second best values \underline{underlined}.}
\label{tab: exp_nerfs}
\begin{tabular}{l|l|ccccccccc|c}
\cline{2-12}
 & Method & Chair & Drums & Ficus & Hotdog & Lego & Materials & Mic & Ship & Avg & Size reduct. \\ \hline
\multirow{3}{*}{\rotatebox[origin=c]{90}{PSNR$\uparrow$}} & Large & \textbf{34.04} & \textbf{24.81} & \textbf{28.84} & \textbf{33.42} & \textbf{29.96} & \textbf{27.01} & \textbf{33.96} & \textbf{22.55} & \textbf{29.32} & - \\
 & Small & 33.12 & \underline{24.14} & 27.77 & 32.06 & 28.75 & \underline{26.47} & \underline{33.68} & \underline{22.28} & 28.53 & \textbf{20.74\%} \\
 & Ours & \underline{33.23} & 24.11 & \underline{27.82} & \underline{33.10} & \underline{28.82} & 26.21 & 33.59 & 22.26 & \underline{28.64} & \textbf{20.74\%} \\ \hline
\end{tabular}
\end{table}
Table~\ref{tab: exp_nerfs} shows that compared to Large, the decrease in PSNR  of \method{} was $13.9\%$ lower than the PSNR of the Small network approach, even when both have the same number of network parameters.
The pruning procedure allows our NeRF to save more than $20\%$ of network parameters compared to the Large approach.
We also see qualitative improvements compared to the Small network, such as shadows/bright spots in the Hotdog (see Figure~\ref{fig:teaser}). 

\end{document}